\documentclass[runningheads, envcountsame, a4paper]{llncs}

\usepackage[pdftex]{graphicx}
\usepackage{epstopdf}
\usepackage{amsmath}
\usepackage{amssymb}
\usepackage{amsfonts}
\usepackage{algorithm}
\usepackage{algorithmic}
\usepackage{multirow}
\usepackage{appendix}
\usepackage[misc]{ifsym}
\usepackage[hyphens]{url}
\usepackage{subfigure}
\usepackage{hyperref}
\hypersetup{colorlinks,allcolors=green}

\newtheorem{bound}{Bound}

\begin{document}
\title{Robust Domain Adaptation: \\ Representations, Weights and Inductive Bias}
%
%
%

\author{Victor Bouvier\inst{1, 2}(\Letter) \and
        Philippe Very\thanks{Work done when author was at Sidetrade} \inst{3} \and
        Cl\'ement Chastagnol$^\star$ \inst{4} \and 
        Myriam Tami\inst{1} \and 
        C\'eline Hudelot\inst{1}}

\authorrunning{V. Bouvier et al.}
           
\institute{Université Paris-Saclay, CentraleSupélec, Mathématiques et Informatique pour la Complexité et les Systèmes, 91190, Gif-sur-Yvette, France \email{firstname.name@centralesupelec.fr}\and
           Sidetrade, 114 Rue Gallieni, 92100, Boulogne-Billancourt, France, \email{vbouvier@sidetrade.com}\and 
           Lend-Rx, 24 Rue Saint Dominique, 75007, Paris, France, \\ \email{philippe.very@lend-rxtech.com}  \and 
           Alan, 117 Quai de Valmy, 75010 Paris, France, \\
           \email{clement.chastagnol@alan.eu}}

%
\maketitle              

\begin{abstract}
Unsupervised Domain Adaptation (UDA) has attracted a lot of attention in the last ten years. The emergence of Domain Invariant Representations (IR) has improved drastically the transferability of representations from a labelled source domain to a new and unlabelled target domain. However, a potential pitfall of this approach, namely the presence of \textit{label shift}, has been brought to light. Some works address this issue with a relaxed version of domain invariance obtained by weighting samples, a strategy often referred to as Importance Sampling. From our point of view, the theoretical aspects of how Importance Sampling and Invariant Representations interact in UDA have not been studied in depth. In the present work, we present a bound of the target risk which incorporates both weights and invariant representations. Our theoretical analysis highlights the role of inductive bias in aligning distributions across domains. We illustrate it on standard benchmarks by proposing a new learning procedure for UDA. We observed empirically that weak inductive bias makes adaptation more robust. The elaboration of stronger inductive bias is a promising direction for new UDA algorithms.
\end{abstract}
\keywords{Unsupervised Domain Adaptation  \and Importance Sampling \and Invariant Representations \and Inductive Bias}

\section{Introduction}
Deploying machine learning models in the real world often requires the ability to generalize to \textit{unseen samples} \textit{i.e.} samples significantly different from those seen during learning. Despite impressive performances on a variety of tasks, deep learning models do not always meet these requirements \cite{beery2018recognition,geva2019we}. For this reason, \textit{out-of-distribution generalization} is recognized as a major challenge for the reliability of machine learning systems \cite{amodei2016concrete,arjovsky2019invariant}.  Domain Adaptation (DA) \cite{quionero2009dataset,pan2009survey} is a well-studied approach to bridge the gap between train and test distributions. In DA, we refer to train and test distributions as \textit{source} and \textit{target} respectively noted $p_S(x,y)$ and $p_T(x,y)$ where $x$ are inputs and $y$ are labels. The objective of DA can be defined as learning a good classifier on a poorly sampled target domain by leveraging samples from a source domain. Unsupervised Domain Adaptation (UDA) assumes that only unlabelled data from the target domain is available during training. In this context, a natural assumption, named \textit{Covariate shift} \cite{shimodaira2000improving,huang2007correcting}, consists in assuming that the mapping from the inputs to the labels is conserved across domains, \textit{i.e.} $p_T(y|x) = p_S(y|x)$.  In this context, \textit{Importance Sampling} (IS) performs adaptation by weighting the contribution of sample $x$ in the loss by $w(x) = p_T(x) / p_S(x)$ \cite{quionero2009dataset}.  Although IS seems natural when unlabelled data from the target domain is available, the covariate shift assumption is not sufficient to guarantee successful adaptation \cite{ben2007analysis}. Moreover, for high dimensional data \cite{d2017overlap} such as texts or images, the shift between $p_S(x)$ and $p_T(x)$ results from non-overlapping supports leading to unbounded weights \cite{johansson2019support}. 

In this particular context, representations can help to reconcile non-overlapping supports \cite{ben2007analysis}. This seminal idea, and the corresponding theoretical bound of the target risk from \cite{ben2007analysis}, has led to a wide variety of deep learning approaches \cite{ganin2015unsupervised,long2015learning,long2018conditional} which aim to learn a so-called \textit{domain invariant representation}:
\begin{equation}
    p_S(z) \approx p_T(z)
    \label{domain_invariance}
\end{equation}
where $z:=\varphi(x)$ for a given non-linear representation $\varphi$. These assume that the \textit{transferability} of representations, defined as the combined error of an ideal classifier, remains low during learning. Unfortunately, this quantity involves target labels and is thus intractable. More importantly, looking for strict invariant representations, $p_S(z) = p_T(z)$, hurts the transferability of representations \cite{johansson2019support,liu2019transferable,wu2019domain,zhao2019learning}. In particular, there is a fundamental trade-off between learning invariant representations and preserving transferability in presence of label shift ($p_T(y) \neq p_S(y)$)  \cite{zhao2019learning}. To mitigate this trade-off, some recent works suggest to relax domain invariance by weighting samples \cite{cao2018unsupervised,wu2019domain,you2019universal,cao2018partial}. This strategy differs with  (\ref{domain_invariance}) by aligning a \textit{weighted source} distribution with the target distribution:
\begin{equation}
    w(z) p_S(z) \approx p_T(z)
\end{equation}
for some weights $w(z)$. We now have two tools, $w$ and $\varphi$, which need to be calibrated to obtain distribution alignment. Which one should be promoted? How weights preserve good transferability of representations? 

While most prior works focus on the invariance error for achieving adaptation \cite{ganin2015unsupervised,long2015learning,long2018conditional}, this paper focuses on the transferability of representations.  We show that weights allow to design an interpretable generalization bound where transferability and invariance errors are uncoupled.  In addition, we discuss the role of inductive design for both the classifier and the weights in addressing the lack of labelled data in the target domain.  Our contributions are the following:

\begin{enumerate}
    \item We introduce a new bound of the target risk which incorporates both weights and domain invariant representations. Two new terms are introduced. The first is an \textit{invariance error}  which promotes alignment between a weighted source distribution of representations and the target distribution of representations. The second, named \textit{transferability error}, involves labelling functions from both source and target domains.
    
    \item We highlight the role of \textbf{inductive bias} for approximating the transferability error. First, we establish connections between our bound and popular approaches for UDA which use target predicted labels during adaptation, in particular Conditional Domain Adaptation \cite{long2018conditional} and Minimal Entropy \cite{grandvalet2005semi}. Second, we show that the inductive design of weights has an impact on representation invariance.
    \item We derive a new learning procedure for UDA. The particularity of this procedure is to only minimize the transferability error while controlling representation invariance with weights. Since the transferability error involves target labels, we use the predicted labels  during learning.
    \item  We provide an empirical illustration of our framework on two DA benchmarks (\textbf{Digits} and \textbf{Office31} datasets). We stress-test our learning scheme by modifying strongly the label distribution in the source domain. While methods based on invariant representations deteriorate considerably in this context, our procedure remains robust. 
\end{enumerate}

\section{Preliminaries}
We introduce the \textit{source} distribution \textit{i.e.} data where the model is trained with supervision and the \textit{target} distribution \textit{i.e.} data where the model is tested or applied. Formally, for two random variables $(X,Y)$ on a given space $\mathcal X \times \mathcal Y$, we introduce two distributions: the source distribution $p_S(x,y)$ and the target distribution $p_T(x,y)$. Here, labels are one-hot encoded \textit{i.e.} $y \in [0,1]^{C}$ such that $\sum_c y_c = 1$ where $C$ is the number of classes. The distributional shift situation is then characterized by $p_S(x,y) \neq p_T(x,y)$ \cite{quionero2009dataset}. In the rest of the paper, we use the index notation $S$ and $T$ to differentiate source and target terms.  We define the hypothesis class $\mathcal H$ as a subset of functions from $\mathcal X$ to $\mathcal Y$ which is the composition of a representation class $\Phi$ and a classifier class $\mathcal G$, \textit{i.e.} $\mathcal H= \mathcal G \circ \Phi$. For the ease of reading, given a classifier $g \in \mathcal G$ and a representation $\varphi \in \Phi$, we note $g\varphi := g\circ\varphi$. Furthermore, in the definition $z:=\varphi(x)$, we refer indifferently to $z$, $\varphi$, $Z:=\varphi(X)$ as the \textit{representation}. For two given $h$ and $h'\in \mathcal H$ and $\ell$ the $L^2$ loss $\ell(y,y') = ||y-y'||^2$, the risk in domain $D \in \{S,T\}$ is noted:

\begin{equation}
    \varepsilon_D(h) := \mathbb E_D[\ell(h(X),Y)] 
\end{equation} 
and $\varepsilon_D(h,h') := \mathbb E_D[\ell(h(X), h'(X))]$. In the seminal works \cite{ben2007analysis,mansour2009domain}, a theoretical limit of the target risk when using a representation $\varphi$ has been derived:
\begin{bound}[Ben David et al.]
\label{ben_david}
Let $d_{\mathcal G}(\varphi) = \sup_{g,g' \in \mathcal G}|\varepsilon_S(g \varphi, g'\varphi) - \varepsilon_T(g \varphi, g'\varphi)|$ and $\lambda_{\mathcal G}(\varphi) =  \inf_{g\in \mathcal G}\{ \varepsilon_S(g \varphi) + \varepsilon_T(g\varphi)\}$, $\forall g \in \mathcal G, \forall \varphi \in \Phi$:
\begin{equation}
    \varepsilon_T(g\varphi)\leq \varepsilon_S(g\varphi) + d_{\mathcal G}(\varphi)  + \lambda_{\mathcal G}(\varphi)
\end{equation}
\end{bound}
This generalization bound ensures that the target risk $\varepsilon_T(g\varphi$) is bounded by the sum of the source risk $\varepsilon_S(g\varphi$), the disagreement risk between two classifiers from representations $d_{\mathcal G}(\varphi)$, and a third term, $\lambda_{\mathcal G}(\varphi)$, which quantifies the ability to perform well in both domains from representations. The latter is referred to as the \textit{adaptability} error of representations. It is intractable in practice since it involves labels from the target distribution. Promoting distribution invariance of representations, \textit{i.e.} $p_S(z)$ close to $p_T(z)$, results on a low $d_{\mathcal G}(\varphi)$. More precisely: 
\begin{equation}
    d_\mathcal G(\varphi) \leq 2 \sup_{d \in \mathcal D} |p_S(d(z)=1) - p_T(d(z)=0)|
    \label{bound_discriminator}
\end{equation} 
where $\mathcal D$ is the so-called set of \textit{discriminators} or \textit{critics} which verifies $\mathcal D \supset \{g \oplus g': (g,g') \in \mathcal G^2 \}$ where $\oplus$ is the $\mathrm{XOR}$ function \cite{ganin2015unsupervised}. Since the domain invariance term $d_{\mathcal G}(\varphi)$ is expressed as a supremal value on classifiers, it is suitable for domain adversarial learning with critic functions. Conversely, the adaptability error $\lambda_{\mathcal G}(\varphi)$ is expressed as an infremal value. This '$\sup / \inf$' duality induces an unexpected trade-off when learning domain invariant representations: 

\begin{proposition}[Invariance hurts adaptability \cite{johansson2019support,zhao2019learning}]
\label{invariance_hurts}
Let $\psi$ be a representation which is a richer feature extractor than $\varphi$: $\mathcal G\circ\varphi \subset \mathcal G \circ \psi$. Then,
\begin{equation}
    d_{\mathcal G}(\varphi) \leq d_{\mathcal G}(\psi) \mbox{ while } \lambda_{\mathcal G}(\psi) \leq \lambda_{\mathcal G}(\varphi)
    \label{BD_eq}
\end{equation}
\end{proposition}
As a result of proposition \ref{invariance_hurts}, the benefit of representation invariance must be higher than the loss of adaptability, which is impossible to guarantee in practice.

\section{Theory}

\label{theory}
To overcome the limitation raised in proposition \ref{invariance_hurts}, we expose a new bound of the target risk which embeds a new trade-off between invariance and transferability  (\ref{a_new_trade_off}). We show this new bound remains inconsistent with the presence of label shift (\ref{detailed_view_of_tightnes}) and we expose the role of weights to address this problem (\ref{reconciling_weights}).

\subsection{A new trade-off between Invariance and Transferability}
\label{a_new_trade_off}

\subsubsection{Core assumptions.} Our strategy is to express  both the transferability and invariance as a supremum using Integral Probability Measure (IPM) computed on a critic class.  We thus introduce a class of critics suitable for our analysis. Let $\mathcal F$  from $\mathcal Z\to [-1,1]$ and $\mathcal F_C$  from $\mathcal Z\to [-1,1]^{C}$ with the following properties:
\begin{itemize}
    \item (A1) $\mathcal F$ and $\mathcal F_C$ are symmetric (\textit{i.e.} $\forall f \in \mathcal F, -f \in \mathcal F$) and convex.
    \item (A2) $\mathcal G \subset \mathcal F_C$ and $\left  \{\mathbf f \cdot \mathbf f'~ ; ~\mathbf f, \mathbf f' \in \mathcal F_C  \right \} \subset \mathcal F$.
    \item (A3) $\forall \varphi \in \Phi$, $\mathbf f_D(z) \mapsto \mathbb E_D[Y|\varphi(X)=z] \in \mathcal F_C$. \footnote{See Appendix \ref{proof:TVBound_no_w} for more details on this assumption.}
    \item (A4) For two distributions $p$ and $q$ on $\mathcal Z$, $p=q$ if and only if:
    \begin{equation}
        \mathrm{IPM}(p,q ; \mathcal F) := \sup_{f \in \mathcal F} \left \{ \mathbb E_p[f(Z)] - \mathbb E_q[f(Z)] \right \} = 0
    \end{equation}
\end{itemize}
The assumption (A1) ensures that rather comparing two given $\mathbf f$ and $\mathbf f'$, it is enough to study the error of some $\mathbf f'' = \frac{1}{2} (\mathbf f - \mathbf f')$ from $\mathcal F_C$. This brings back a supremum on $\mathcal F_C^2$ to a supremum on $\mathcal F_C$. The assumption (A2), combined with (A1),  ensures that an error $\ell(\mathbf f, \mathbf f') $ can be expressed as a critic function $f \in \mathcal F$ such that $f = \ell(\mathbf f, \mathbf f')$. The assumption (A3) ensures that $\mathcal F_C$ is rich enough to contain label function from representations. Here, $\mathbf f_D(z) = \mathbb E_D[Y|Z=z]$ is a vector of probabilities on classes: $f_D(z)_c = p_D(Y=c|Z=z)$. The last assumption (A4) ensures that the introduced IPM is a distance. Classical tools verify these assumptions \textit{e.g.} continuous functions; here $\mathrm{IPM}(p,q ; \mathcal F)$ is the \textit{Maximum Mean Discrepancy} \cite{gretton2012kernel}  and one can reasonably believe that $\mathbf f_S$ and $\mathbf f_T$ are continuous.

\subsubsection{Invariance and transferability as IPMs.} 

We introduce here two important tools that will guide our analysis:
\begin{itemize}
    \item $\mathrm{INV}(\varphi)$, named \textit{invariance error}, that aims at capturing the difference between source and target distribution of representations, corresponding to:
    \begin{equation}
        \mathrm{INV}(\varphi)  := \sup_{f\in \mathcal F}  \left \{ \mathbb E_T[f(Z)] - \mathbb E_{S}[f(Z)] \right \}
    \end{equation}
    \item $\mathrm{TSF}(\varphi)$, named \textit{transferability error}, that catches if the coupling between $Z$  and $Y$ shifts across domains. For that, we use our class of functions $\mathcal F_C$ and we compute the IPM of $Y\cdot \mathbf f(Z)$, where $\mathbf f \in \mathcal F_C$ and $Y\cdot \mathbf f(Z)$ is the scalar product\footnote{the scalar product between $Y$ and $\mathbf f (Z)$ emerges from the choice of the $L^2$ loss.}, between the source and the target domains:
    \begin{equation}
        \mathrm{TSF}(\varphi) := \sup_{\mathbf f\in \mathcal F_C} \{ \mathbb E_T[Y \cdot \mathbf f (Z)] - \mathbb E_{S}[Y \cdot \mathbf f (Z)] \}
    \end{equation}
\end{itemize}

\subsubsection{A new bound of the target risk.} Using $\mathrm{INV}(\varphi)$ and $\mathrm{TSF}(\varphi)$, we can provide a new bound of the target risk:
\begin{bound}
\label{TVBound_no_w}
$\forall g \in \mathcal G$ and $\forall \varphi \in \Phi$:
\begin{equation}
    \varepsilon_T(g\varphi) \leq \varepsilon_{S}(g\varphi)  + 6\cdot\mathrm{INV}(\varphi) +  
    2\cdot \mathrm{TSF}(\varphi)  + \varepsilon_T(\mathbf f_T\varphi)
\end{equation}
\end{bound}
The proof is in Appendix \ref{proof:TVBound_no_w}. In contrast with bound \ref{ben_david} (Eq. \ref{BD_eq}),  here two IPMs are involved to compare representations ($\mathrm{INV}(\varphi)$ and $\mathrm{TSF}(\varphi)$).  A new term, $\varepsilon_T(\mathbf f_T\varphi)$, reflects the level of noise when fitting labels from representations. All the trade-off between invariance and transferability is embodied in this term:
\begin{proposition}
Let $\psi$ a representation which is a richer feature extractor than $\varphi$: $\mathcal F\circ\varphi \subset \mathcal F \circ \psi$ and $\mathcal F_C\circ\varphi \subset \mathcal F_C \circ \psi$. $\varphi$ is more domain invariant than $\psi$:
\begin{equation}
    \mathrm{INV}(\varphi) \leq \mathrm{INV}(\psi) \mbox{ while } \varepsilon_T (\mathbf f_T^\psi\psi ) \leq \varepsilon_T  (\mathbf f_T^\varphi \varphi)
\end{equation}
where $\mathbf f_T^\varphi(z) = \mathbb E_T[Y|\varphi(X) = z]$ and $\mathbf f_T^\psi(z) = \mathbb E_T[Y|\psi(X)=z]$. Proof in \ref{proof:new_trade_off}.
\end{proposition}
 Bounding the target risk using IPMs has two advantages. First, it allows to better control the invariance / transferability trade-off since $\varepsilon_T(\mathbf f_T\varphi) \leq \lambda_{\mathcal G}(\varphi)$. This is paid at the cost of $4 \cdot \mathrm{INV}(\varphi) \geq d_{\mathcal G}(\varphi)$ (see Proposition \ref{d_c_inv} in Appendix \ref{proof:TVBound_no_w}). Second, $\varepsilon_T(\mathbf f_T\varphi)$ is source free and indicates whether there is enough information in representations for learning the task in the target domain at first. This means that $\mathrm{TSF}(\varphi)$ is only dedicated to control if aligned representations have the same labels across domains. To illustrate the interest of our new transferability error, we provide visualisation of representations (Fig. \ref{fig:transferability_tsne}) when trained to minimize the adaptability error $\lambda_{\mathcal G}(\varphi) $ from bound \ref{ben_david} and the transferability error $\mathrm{TSF}(\varphi)$ from bound \ref{TVBound_no_w}.

\begin{figure*}[t!]
  \centering
  \subfigure[$\lambda_{\mathcal G}(\varphi)$ adaptability in bound \ref{ben_david} from \cite{ben2007analysis}. Inside class clusters, source and target representations are separated.]{
    \includegraphics[width=0.47\textwidth]{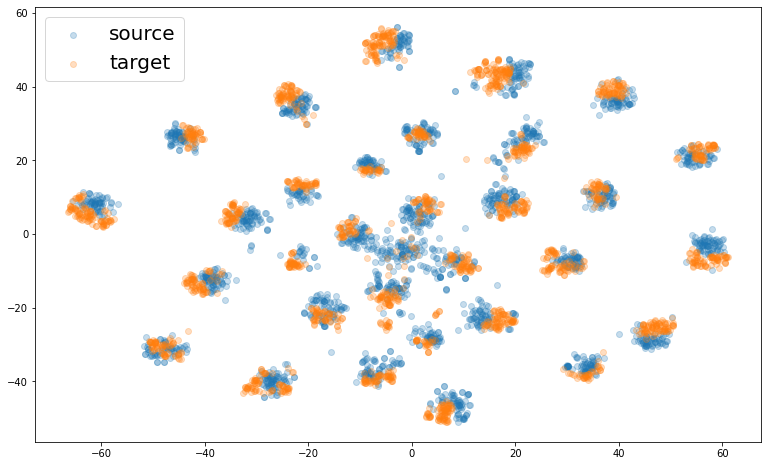}
  }
  \subfigure[$\mathrm{TSF}(\varphi)$ transferability from bound \ref{TVBound_no_w} (contribution). Inside class clusters, source and target representations are not distinguishable]{
    \includegraphics[width=0.47\textwidth]{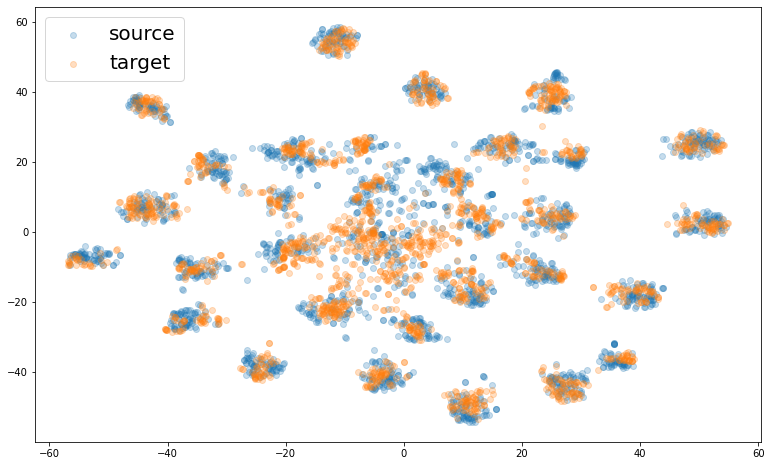}
  }
  \caption{t-SNE \cite{maaten2008visualizing} visualisation of representations when trained to minimize (a) adaptability error $\lambda_{\mathcal G}(\varphi)$ from \cite{ben2007analysis}, (b) transferability error $\mathrm{TSF}(\varphi)$ introduced in the present work. The task used is A$\to$W of the \textbf{Office31} dataset. \textit{Labels in the target domain are used during learning in this specific experiment}. For both visualisations of representations, we observe well-separated clusters associated to the label classification task. Inside those clusters, we observe a separation between source and target representations for $\lambda_{\mathcal G}(\varphi)$. That means that representations embed domain information and thus are not invariant. On the contrary, source and target representations are much more overlapping inside of each cluster with  $\mathrm{TSF}(\varphi)$, illustrating that this new term is not conflictual with invariance.}
  \vspace{-0.2in}
  \label{fig:transferability_tsne}
\end{figure*}

\subsection{A detailed view on the property of tightness}


\label{detailed_view_of_tightnes}
 An interesting property of the bound, named tightness,
is the case when $\mathrm{INV}(\varphi)=0$ and $\mathrm{\mathrm{TSF}(\varphi)} =0$ simultaneously. The condition of tightness of the bound provides rich information on the properties of representations. 
\begin{proposition}
\label{tight_no_w}
$\mathrm{INV}(\varphi) = \mathrm{TSF}(\varphi) = 0$ if and only if $
    p_S(y,z) = p_T(y,z)$.
\end{proposition}
The proof is given in Appendix \ref{proof:tightness_TVBound_no_w}. Two important points should be noted:
\begin{enumerate}
    \item  $\mathrm{INV}(\varphi) = 0$ ensures that  $p_S(z) = p_T(z)$, using (A4). Similarly, $\mathrm{TSF(\varphi)} =0 $ leads to $p_S(y,z)= p_T(y,z)$. Since $p_S(y,z)= p_T(y,z)$ implies $p_S(z)= p_T(z)$, $\mathrm{INV}(\varphi)$ does not bring more substantial information about representations distribution than $\mathrm{TSF}(\varphi)$. More precisely, one can show that $\mathrm{TSF}(\varphi ) \geq \mathrm{INV}(\varphi)$ noting that $Y\cdot \mathbf f(Z) = f(z)$ when $\mathbf f(z) = (f(z), ..., f(z))$ for $f\in\mathcal F$.
    \item Second, the equality $p_S(y,z)=  p_T(y,z)$ also implies that $p_S(y)=p_T(y)$. Therefore, in the context of label shift (when $p_S(y) \neq p_T(y)$), the transferability error cannot be null. This is a big hurdle since it is clearly established that most real world UDA tasks exhibit some label shift. This bound highlights the fact that representation invariance alone can not address UDA in complex settings such as the label shift one. 
\end{enumerate}

\subsection{Reconciling Weights and Invariant Representations.}  
\label{reconciling_weights}
Based on the interesting observations from \cite{johansson2019support,zhao2019learning} and following the line of study that proposed to relax invariance using weights \cite{cao2018partial,zhang2018importance,you2019universal,wu2019domain}, we propose to adapt the bound by incorporating weights. More precisely, we study the effect of modifying the source distribution $p_S(z)$ to a \textit{weighted source} distribution $w(z) p_S(z)$ where $w$ is a positive function which verifies $\mathbb E_S[w(Z)] = 1$. By replacing $p_S(z)$ by $w(z)p_S(z)$ (distribution referred as $w\cdot S$) in bound \ref{TVBound_no_w}, we obtain a new bound of the target risk incorporating both weights and representations:

\begin{bound}
\label{TV_bound_with_w}
$\forall g \in \mathcal G, \forall w: \mathcal Z\to \mathbb R^+$ such that $\mathbb E_S[w(z)]=1$:
\begin{align}
    \notag \varepsilon_T(g\varphi) \leq \varepsilon_{w\cdot S}(g\varphi)  +6 \cdot \mathrm{INV}(w, \varphi) +  
    2 \cdot \mathrm{TSF}(w, \varphi)  + \varepsilon_T(\mathbf f_T \varphi)
\end{align}
where $ \mathrm{INV}(w, \varphi) := \displaystyle{\sup_{f\in \mathcal F}}  \left \{ \mathbb E_T[f(Z)] - \mathbb E_{S}[w(Z)f(Z)] \right \} $ and $\mathrm{TSF}(w, \varphi)  := \displaystyle{\sup_{\mathbf f\in \mathcal F_{C}}}\{ $ $\mathbb E_T[Y \cdot \mathbf f(Z)]    - \mathbb E_{S}[w(Z) Y \cdot \mathbf f(Z)] \}$.
\end{bound}
As for the previous bound \ref{TVBound_no_w}, the property of tightness, \textit{i.e.} when invariance and transferability are null simultaneously, leads to interesting observations:
\begin{proposition}
\label{tight}
$\mathrm{INV}(w,\varphi) = \mathrm{TSF}(w,\varphi) = 0$ if and only if $
    w(z) = \frac{p_T(z)}{p_S(z)}$ and $\mathbb E_T[Y|Z=z] = \mathbb E_S[Y|Z=z]$. The proof is given in Appendix \ref{proof:tightness_TV_bound_with_w}.
\end{proposition}
This proposition means that the nullity of invariance error, \textit{i.e.} $\mathrm{INV}(w,\varphi) = 0$, implies distribution alignment, \textit{i.e.} $w(z) p_S(z) = p_T(z)$. This is of strong interest since both representations and weights are involved for achieving domain invariance. The nullity of the transferability error, \textit{i.e.} $\mathrm{TSF}(w,\varphi) =0$, implies that labelling functions, $\mathbf f:z \mapsto \mathbb E[Y|Z=z]$, are conserved across domains. Furthermore, the equality $\mathbb E_T[Y|Z] = \mathbb E_S[Y|Z]$ interestingly resonates with a recent line of work called \textit{Invariant Risk Minimization} (IRM) \cite{arjovsky2019invariant}. Incorporating weights in the bound thus brings two benefits: 
\begin{enumerate}
    \item First, it raises the inconsistency issue of invariant representations in presence of label shift, as mentioned in section \ref{tight_no_w}. Indeed, tightness is not conflicting with label shift.
    \item $\mathrm{TSF}(w, \varphi)$ and $\mathrm{INV}(w,\varphi)$ have two disctinct roles: the former promotes domain invariance of representations while the latter controls whether aligned representations share the same labels across domains.
\end{enumerate}

\section{The role of Inductive Bias}
\textit{Inductive Bias} refers to the set of assumptions which improves generalization of a model trained on an empirical distribution. For instance, a specific neural network architecture or a well-suited regularization are prototypes of inductive biases. First, we provide a theoretical analysis of the role of inductive bias for addressing the lack of labelling data in the target domain (\ref{inductive_classifier}), which is the most challenging part of \textit{Unsupervised} Domain Adaptation. Second, we describe the effect of weights to induce invariance property on representations (\ref{inductive_weights}).

\subsection{Inductive design of a classifier} 
\label{inductive_classifier}
\subsubsection{General Formulation.} Our strategy consists in approximating target labels error through a classifier $\tilde g \in \mathcal G$. We refer to the latter as the inductive design of the classifier. Our proposition follows the intuitive idea  which states that the best source classifier, $g_S:= \arg \min_{g\in \mathcal G} \varepsilon_S(g\varphi)$, is not necessarily the best target classifier \textit{i.e.} $g_S \neq \arg  \min_{g\in \mathcal G} \varepsilon_T(g\varphi)$. For instance, a well-suited regularization in the target domain, noted $\Omega_T(g)$ may improve performance, \textit{i.e.} setting $\tilde g := \arg \min_{g\in \mathcal G} \varepsilon_S(g\varphi) + \lambda \cdot \Omega_T(g)$ may lead to $\varepsilon_T(\tilde g\varphi) \leq \varepsilon_T(g_S\varphi)$. We formalize this idea through the following definition:
\begin{definition}[Inductive design of a classifier] We say that there is an inductive design of a classifier at level $0 < \beta \leq 1$ if for any representations $\varphi$, noting $g_S = \arg \min_{g \in \mathcal G} \varepsilon_S (g \varphi)$, we can determine $\tilde g$ such that: 
\begin{equation}
    \varepsilon_T(\tilde g\varphi) \leq \beta \varepsilon_T(g_S\varphi)
\end{equation}
We say the inductive design is $\beta-$strong when $\beta <1$ and weak when $\beta=1$.
\end{definition}
In this definition, $\beta$ does not depend of $\varphi$, which is a strong assumption, and embodies the strength of the inductive design. The closer to 1 is $\beta$, the less improvement we can expect using the inductive classifier $\tilde g$. We now study the impact of the inductive design of a classifier in our previous bound \ref{TV_bound_with_w}. Thus, we introduce the approximated transferability error:
\begin{equation}
    \widehat{\mathrm{TSF}}(w, \varphi, \tilde g) = \sup_{\mathbf f\in \mathcal F_C} \{ \mathbb E_T[\tilde  g(Z) \cdot \mathbf f(Z)] - \mathbb E_{S}[w(Z) Y \cdot \mathbf f(Z)] \}
\end{equation}
leading to a bound of the target risk where transferability is target labels free:
\begin{bound}[Inductive Bias and Guarantee]
\label{IB}
Let $\varphi \in \Phi$ and $w: \mathcal Z\to \mathbb R^+$ such that $\mathbb E_S[w(z)]=1$ and a $\beta-$strong inductive classifier $\tilde g$ and $\rho := \frac{\beta}{1- \beta}$ then:
\begin{equation}
    \varepsilon_T(\tilde g\varphi) \leq \rho \left ( \varepsilon_{w\cdot S}(g_{w\cdot S}\varphi)  + 6\cdot \mathrm{INV}(w, \varphi) +  
    2\cdot  \widehat{\mathrm{TSF}}(w, \varphi, \tilde g)  + \varepsilon_T(\mathbf f_T\varphi)\right) 
\end{equation}
\end{bound}
The proof is given in Appendix \ref{proof:IB}. Here, the target labels are only involved in $\varepsilon_T(\mathbf f_T\varphi)$ which reflects the level of noise when fitting labels from representations. Therefore, transferability is now free of target labels. This is an important result since the difficulty of UDA lies in the lack of labelled data in the target domain. It is also interesting to note that the weaker the inductive bias ($\beta \to 1$), the higher the bound and vice versa.




\subsubsection{The role of predicted labels.} Predicted labels play an important role in UDA. In light of the inductive classifier, this means that $\tilde g$ is simply set as $g_{w\cdot S}$. This is a weak inductive design ($\beta=1$), thus, theoretical guarantee from bound \ref{IB} is not applicable. However, there is empirical evidence that showed that predicted labels help in UDA \cite{grandvalet2005semi,long2018conditional}. It suggests that this inductive design may find some strength in the finite sample regime. A better understanding of this phenomenon is left for future work (See Appendix \ref{open_dicussion}). In the rest of the paper, we study this weak inductive bias by establishing connections between $\widehat{\mathrm{TSF}}(w, \varphi, g_S)$ and popular approaches of the literature.

\paragraph{Connections with Conditional Domain Adaptation Network.} CDAN \cite{long2018conditional} aims to align the joint distribution $(\hat Y, Z)$ across domains, where $\hat Y = g_S \varphi(X)$ are estimated labels. It is performed by exposing the tensor product between $\hat Y$ and $Z$ to a discriminator. It leads to substantial empirical improvements compared to \textit{Domain Adversarial Neural Networks} (DANN) \cite{ganin2015unsupervised}. We can observe that it is a similar objective to $\widehat{\mathrm{TSF}}(w,\varphi, g_S)$ in the particular case where $w(z)=1$.

\paragraph{Connections with Minimal Entropy.} MinEnt \cite{grandvalet2005semi} states that an adapted classifier is confident in prediction on target samples. It suggests the regularization: $\Omega_T(g) := H(\hat Y|Z) = \mathbb E_{Z \sim p_T}[-g(Z) \cdot \log g(Z)]$ where $H$ is the entropy. If labels are smooth enough (\textit{i.e.} it exists $\alpha$ such that $\frac{\alpha}{C-1} \leq \mathbb E_S[\hat Y |Z] \leq 1 -\alpha$), MinEnt is a lower bound of transferability: $\widehat{\mathrm{TSF}}(w,\varphi, g_s) \geq \eta \left (H_T(g_S\varphi) - \mathrm{CE}_{w\cdot S}(Y, g_S\varphi)\right)$ for some $\eta >0$ and $\mathrm{CE}_{w\cdot S}(g_S\varphi, Y)$ is the cross-entropy between $g_S\varphi$ and $Y$ on $w(z) p_S(z)$ (see Appendix \ref{proof:MinENT}). 

\subsection{Inductive design of weights}
\label{inductive_weights}
While the bounds introduced in the present work involve weights in the representation space, there is an abundant literature that builds weights in order to relax the domain invariance of representations \cite{cao2018unsupervised,wu2019domain,you2019universal,cao2018partial}. We study the effect of inductive design of $w$ on representations. To conduct the analysis, we consider there is a non-linear transformation $\psi$ from $\mathcal Z$ to $\mathcal Z'$ and we assume that weights are computed in $\mathcal Z'$, \textit{i.e.} $w$ is a function of $z':= \psi(z) \in \mathcal Z'$. We refer to this as \textit{inductive design of weights}. For instance, in the particular case where $\psi = g_S$, weights are designed as $w(\hat y) = p_T(\hat Y = \hat y) / p_S(\hat Y = \hat y)$ \cite{cao2018partial} where $\hat Y = g\varphi(X)$. In \cite{long2018conditional}, \textit{entropy conditioning} is introduced by designing weights $ w(z') \propto 1+ e^{-z'}$ where $z' = - \frac{1}{C}\sum_{1\leq c \leq C} g_{S,c} \log(g_{S,c})$ is the predictions entropy. The inductive design of weights imposes invariance property on representations:
\begin{proposition}[Inductive design of $w$ and invariance] Let $\psi: \mathcal Z \to \mathcal Z'$ such that $\mathcal F \circ \psi \subset \mathcal F$ and $\mathcal F_C\circ \psi\subset \mathcal F_C$. Let $w:\mathcal Z' \to \mathbb R^+$ such that $\mathbb E_S[w(Z')]=1$ and we note $Z':= \psi(Z)$. Then, $\mathrm{INV}(w, \varphi) = \mathrm{TSF}(w,\varphi) = 0$ if and only if:
\begin{equation}
    w(z') = \frac{p_T(z')}{p_S(z')} ~~\mbox{ and } ~~ p_S(z|z') = p_T(z|z')
\end{equation}
while both $\mathbf f_S^\varphi = \mathbf f_T^\varphi$ and $\mathbf f_S^\psi = \mathbf f_T^\psi$. The proof is given in Appendix \ref{proof:inductive_weight}.
\end{proposition}
This proposition shows that the design of $w$ has a significant impact on the property of domain invariance of representations. Furthermore, both labelling functions are conserved. In the rest of  the paper we focus on weighting in the representation space which consists in:
\begin{equation}
    w(z) = \frac{p_T(z)}{p_S(z)}
\end{equation}
Since it does not leverage any transformations of representations $\psi$, we refer to this approach as a weak inductive design of weights. It is worth noting this inductive design controls naturally the invariance error \textit{i.e.} $\mathrm{INV}(w, \varphi) =0 $.

\section{Towards Robust Domain Adaptation}
\label{algorithm}

In this section, we expose a new learning procedure which relies on weak inductive design of both weights and the classifier. This procedure focuses on the transferability error since the inductive design of weights naturally controls the invariance error. Our learning procedure is then a bi-level optimization problem, named RUDA (Robust UDA): 
\begin{equation} \displaystyle{
\left \{ 
\begin{array}{rl}
 \varphi^\star & = \displaystyle{\arg \min_{\varphi\in\Phi} } ~~ \varepsilon_{w(\varphi) \cdot S} (g_{w\cdot S} \varphi) + \lambda \cdot \widehat{\mathrm{TSF}}(w,\varphi, g_{w \cdot S} ) \\
     & \mbox{such that } w(\varphi) = \displaystyle{\arg \min_w} ~ \mathrm{INV}(w, \varphi) \\
     
\end{array}\right.}
    \tag{RUDA}
\end{equation}
\noindent where $\lambda >0$ is a trade-off parameter. Two discriminators are involved here. The former is a domain discriminator $d$ trained to map 1 for source representations and 0 for target representations by minimizing a domain adversarial loss: 
\begin{equation}
    \mathcal L_{\mathrm{INV}}(\theta_d |\theta_\varphi) =  \frac{1}{n_S} \sum_{i=1}^{n_S}- \log(d(z_{S,i})) + \frac{1}{n_T}\sum_{i=1}^{n_T} - \log(1-d(z_{T,i}))
\end{equation}
where $\theta_d$ and $\theta_\varphi$ are respectively the parameters of $d$ and $\varphi$,  and $n_S$ and $n_T$ are respectively the number of samples in the source and target domains. Setting weights $w_d(z) := (1-d(z)) / d(z)$ ensures that $\mathrm{INV}(w,\varphi)$ is minimal (See Appendix \ref{training_details:relaxed_weights}). The latter, noted $\mathbf d$, maps representations to the label space $[0,1]^C$ in order to obtain a proxy of the transferability error expressed as a domain adversarial objective (See Appendix \ref{tv_to_da}):
\begin{align}
\mathcal L_{\mathrm{TSF}}(\theta_{\varphi},\theta_{\mathbf d} |\theta_d, \theta_g) =   \inf_{\mathbf d} & \left \{ \frac{1}{n_S} \sum_{i=1}^{n_S} -  w_d(z_{S,i}) g(z_{S,i}) \cdot \log(\mathbf d(z_{S,i})) \right.  \notag \\
&~~~~  + \left. \frac{1}{n_T}\sum_{i=1}^{n_T} - g(z_{T,i}) \cdot \log(1-\mathbf d (z_{T,i})) \right \}
\end{align}
where $\theta_{\mathbf d}$ and $\theta_g$ are respectively parameters of $\mathbf d$ and $g$. Furthermore, we use the cross-entropy loss in the source weighted domain for learning $\theta_g$:
\begin{equation}
     \mathcal L_{c}(\theta_g, \theta_\varphi |\theta_d)  = \frac{1}{n_S}\sum_{i=1}^{n_S} - w_d (z_{S,i}) y_{S,i} \cdot \log (g(z_{S,i}))
\end{equation} 
Finally, the optimization is then expressed as follows:
\begin{equation}
\left \{
    \begin{array}{rl}
        \theta_\varphi^\star & = \arg \min_{\theta_\varphi} \mathcal L_{c}(\theta_g, \theta_\varphi |\theta_d) + \lambda \cdot \mathcal L_{\mathrm{TSF}}(\theta_{\varphi},\theta_{\mathbf d} |\theta_d, \theta_g)  \\
         \theta_g  &= \arg \min_{\theta_g}  \mathcal L_{c}(\theta_g,  \theta_\varphi | \theta_d) \\
         \theta_d  &= \arg \min_{\theta_d}  \mathcal L_{\mathrm{INV}}(\theta_d |\theta_\varphi) 
    \end{array}
\right.
\end{equation}
Losses are minimized by stochastic gradient descent (SGD) where in practice $\inf_{d}$ and $\inf_{\mathbf d}$ are gradient reversal layers \cite{ganin2015unsupervised}. The trade-off parameter $\lambda$ is pushed from 0 to $1$ during training. We provide an implementation in \texttt{Pytorch} \cite{paszke2019pytorch} based on \cite{long2018conditional}. The algorithm procedure is described in Appendix \ref{procedure_detailed}.

\section{Experiments}
\label{experiments}

\subsection{Setup}
\paragraph{Datasets.} We investigate two digits datasets: \textbf{MNIST} and \textbf{USPS} transfer tasks MNIST to USPS (M$\to$U) and USPS to MNIST (U$\to$M). We used standard train / test split for training and evaluation. \textbf{Office-31} is a dataset of images containing objects spread among 31 classes captured from different domains: \textbf{Amazon}, \textbf{DSLR} camera and a \textbf{Webcam} camera. \textbf{DSLR} and \textbf{Webcam} are very similar domains but images differ by their exposition and their quality. 


\paragraph{Label shifted datasets.} We stress-test our approach by investigating more challenging settings where the label distribution shifts strongly across domains. For the \textbf{Digits} dataset, we explore a wide variety of shifts by keeping only $5\%$, $10\%$, $15\%$ and $20\%$ of digits between 0 and 5  of the original dataset (refered as $\% \times [0\sim 5]$). We have investigated the tasks U$\to$M and M$\to$U. For the \textbf{Office-31} dataset, we explore the shift where the object spread in classes 16 to 31 are duplicated 5 times (refered as $5\times [16 \sim 31]$). Shifting distribution in the source domain rather than the target domain allows to better appreciate the drop in performances in the target domain compared to the case where the source domain is not shifted.  

\paragraph{Comparison with the state-of-the-art.} For all tasks, we report results from DANN \cite{ganin2015unsupervised} and CDAN \cite{long2018conditional}. To study the effect of weights, we name our method RUDA when weights are set to 1, and RUDA$_w$ when weights are used. For the non-shifted datasets, we report a weighted version of CDAN (entropy conditioning CDAN+E \cite{long2018conditional}). For the label shifted datasets, we report IWAN \cite{zhang2018importance}, a weighted DANN where weights are learned from a second discriminator, and CDAN$_w$ a weighted CDAN where weihghts are added in the same setting than RUDA$_w$.

\paragraph{Training details.} Models are trained during 20.000 iterations of SGD. We report end of training accuracy in the target domain averaged on five random seeds. The model for the \textbf{Office-31} dataset uses a pretrained ResNet-50 \cite{he2016deep}. We used the same hyper-parameters than \cite{long2018conditional} which were selected by importance weighted  cross-validation \cite{sugiyama2007covariate}. The trade-off parameters $\lambda$ is smoothly pushed from 0 to 1 as detailed in \cite{long2018conditional}. To prevent from noisy weighting in early learning, we used weight relaxation: based on the sigmoid output of discriminator $d(z) = \sigma(\tilde d(z))$, we used $d_\tau(z) = \sigma (\tilde d(z/\tau))$ and weights $w(z) = (1 - d_\tau(z))/ d_{\tau}(z)$. $\tau$ is decreased to 1 during training: $\tau = \tau_{\min} + 2(\tau_{\max} - \tau_{\min}) /(1+ \exp(-\alpha p))$ where $\tau_{\max}= 5, \tau_{\min} = 1$, $p \in [0,1]$ is the training progress. In all experiments, $\alpha$ is set to $5$ (except for $5\%\times [0\sim5]$ where $\alpha=15$, see Appendix \ref{ablation_alpha} for more details).

\subsection{Results}
\paragraph{Unshifted datasets.} On both \textbf{Office-31} (Table \ref{table:Office31}) and \textbf{Digits} (Table \ref{table:Digits}), RUDA performs similarly than CDAN. Simply performing the scalar product allows to achieve results obtained by multi-linear conditioning \cite{long2018conditional}. This presents a second advantage: when domains exhibit a large number of classes,  \textit{e.g.} in \textbf{Office-Home} (See Appendix), our approach does not need to leverage a random layer. It is interesting to observe that we achieve performances close to CDAN+E on $\textbf{Office-31}$ while we do not use entropy conditioning. However, we observe a substantial drop in performance when adding weights, but still get results comparable with CDAN in \textbf{Office-31}. This is a deceiptive result since those datasets naturally exhibit label shift; one can expect to improve the baselines using weights. We did not observe this phenomenon on standard benchmarks. 

\begin{table}[t!] 
\centering
\caption{Accuracy ($\%$) on the \textbf{Office-31} dataset.}
\label{table:Office31}
\scriptsize{
\begin{tabular}{|c |c||c|c|c|c|c|c||c|}
\hline
 & Method & A$\overset{}{\to}$W & W$\overset{}{\to}$A & A$\overset{}{\to}$D & D$\overset{}{\to}$A & D$\overset{}{\to}$W & W$\overset{}{\to}$D & Avg  \\
\hline
~~ \multirow{6}{*}{\rotatebox{90}{Standard}} ~~ & ResNet-50 & 68.4 $\pm$ 0.2 & 60.7 $\pm$ 0.3 & 68.9 $\pm$ 0.2 & 62.5 $\pm$ 0.3 & 96.7 $\pm$ 0.1 & 99.3 $\pm$ 0.1  & 76.1  \\ 
& DANN & 82.0 $\pm$ 0.4 & 67.4 $\pm$ 0.5 & 79.7 $\pm$ 0.4 & 68.2 $\pm$ 0.4 & 96.9 $\pm 0.2 $ & 99.1 $\pm 0.1$ & 82.2  \\
& CDAN & 93.1 $\pm$ 0.2 & 68.0 $\pm$ 0.4 & 89.8 $\pm$ 0.3  & 70.1 $\pm$ 0.4 & 98.2 $\pm$ 0.2 & 100. $\pm$ 0.0  & 86.6 \\
& CDAN+E & 94.1 $\pm$ 0.1 & 69.3 $\pm$ 0.4 & \textbf{92.9 $\pm$ 0.2}  & \textbf{71.0 $\pm$ 0.3} & \textbf{98.6 $\pm$ 0.1} & \textbf{100. $\pm$ 0.0} & \textbf{87.7} \\
& RUDA &  \textbf{94.3 $\pm$ 0.3} & \textbf{70.7 $\pm$ 0.3} & 92.1 $\pm$ 0.3 & 70.7 $\pm$ 0.1 & 98.5 $\pm$ 0.1 &  100. $\pm$ 0.0 &  87.6  \\ 
& RUDA$_w$ & 92.0 $\pm$ 0.3 & 67.9 $\pm$ 0.3 & 91.1 $\pm$ 0.3 & 70.2 $\pm$ 0.2 & 98.6 $\pm$ 0.1 & 100. $\pm$ 0.0  &  86.6\\ 
\hline

~~ \multirow{6}{*}{\rotatebox{90}{$5\times [16 \sim 31]$ ~}} ~~  & ResNet-50 & 72.4 $\pm$ 0.7 & 59.5 $\pm$ 0.1 & 79.0 $\pm$ 0.1 & 61.6 $\pm$ 0.3 & 97.8 $\pm$ 0.1 & 99.3 $\pm$ 0.1 &  78.3 \\ 
& DANN & 67.5 $\pm$ 0.1 & 52.1 $\pm$ 0.8 & 69.7 $\pm$ 0.0 & 51.5 $\pm$ 0.1 & 89.9 $\pm 0.1 $ & 75.9 $\pm 0.2$ & 67.8  \\
& CDAN & 82.5 $\pm$ 0.4 & 62.9 $\pm$ 0.6 & \underline{81.4 $\pm$ 0.5}  & \underline{65.5} $\pm$ 0.5 & \underline{98.5 $\pm$ 0.3} & \underline{99.8 $\pm$ 0.0}  & 81.6 \\
& RUDA & \underline{85.4 $\pm$ 0.8} & \underline{66.7 $\pm$ 0.5} & 81.3 $\pm$ 0.3  & 64.0 $\pm$ 0.5 & 98.4 $\pm$ 0.2 &  99.5 $\pm$ 0.1  & \underline{82.1}\\ 
& IWAN & 72.4 $\pm$ 0.4 & 54.8 $\pm$ 0.8 & 75.0  $\pm$ 0.3 & 54.8 $\pm$ 1.3 & 97.0 $\pm 0.0 $ & 95.8 $\pm 0.6$ & 75.0  \\

& CDAN$_w$ & 81.5 $\pm 0.5$ & 64.5 $\pm$ 0.4 & 80.7 $\pm$ 1.0& 65 $\pm$ 0.8  &  \textbf{98.7 $\pm$ 0.2} &  99.9 $\pm$ 0.1  &  81.8 \\
& RUDA$_w$ & \textbf{87.4 $\pm$ 0.2} & \textbf{68.3 $\pm$ 0.3}  & \textbf{82.9} $\pm$ 0.4 & \textbf{68.8 $\pm$ 0.2} & \textbf{98.7} $\pm$ 0.1 & \textbf{100.} $\pm$ 0.0 &  \textbf{83.8} \\
\hline 
\end{tabular}}

\end{table}

\begin{table}[b!]
\centering
\caption{Accuracy ($\%$) on the \textbf{Digits} dataset.}
\label{table:Digits}
\scriptsize{
\begin{tabular}{|c c||ccccc|c||ccccc|c|| c | }
\hline
Method &   & \multicolumn{5}{c}{U$\to$ M } & & \multicolumn{5}{c}{M$\to$U }  & & \\
 & Shift of $[0\sim 5]$ & 5\% & 10\% & 15\% & 20\% & 100\% &  Avg &  5\% & 10\% & 15\% & 20\% & 100\% & Avg & Avg \\
\hline
DANN &  & 41.7 & 51.0 & 59.6 & 69.0 & 94.5 & 63.2 & 34.5 & 51.0  & 59.6 & 63.6 & 90.7 & 59.9 & 63.2  \\
CDAN & & \underline{50.7} & \underline{62.2} & \underline{82.9} & 82.8 & \underline{\textbf{96.9}} &   \underline{75.1} & 32.0 & \underline{69.7} & \underline{78.9} & \underline{81.3} & \underline{\textbf{93.9}} & \underline{71.2} & \underline{73.2} \\  
RUDA & & 44.4 & 58.4 & 80.0 & \underline{84.0} & 95.5 &  72.5  & \underline{34.9} & 59.0 & 76.1 & 78.8 & 93.3 &  68.4 & 70.5\\
IWAN & & 73.7 & 74.4 & 78.4 & 77.5 & 95.7 & 79.9 & 72.2 & 82.0 & 84.3 & 86.0 & 92.0 & 83.3 & 81.6 \\
CDAN$_w$ & & 68.3 & 78.8 & 84.9 & \textbf{88.4} & 96.6 &  83.4  & 69.4 &  80.0 & 83.5 & 87.8 & 93.7 & 82.9 &  83.2 \\ 
RUDA$_w$ & & \textbf{78.7} & \textbf{82.8} & \textbf{86.0} & 86.9 & 93.9 & \textbf{85.7}   &  \textbf{78.7} & \textbf{87.9} & \textbf{88.2} & \textbf{89.3} & 92.5 & \textbf{87.3}&  \textbf{86.5} \\
\hline 
\end{tabular}}
\end{table}

\paragraph{Label shifted datasets.} We stress-tested our approach by applying strong label shifts to the datasets. First, we observe a drop in performance for all methods based on invariant representations compared with the situation without label shift. This is consistent with works that warn the pitfall of domain invariant representations in presence of label shift \cite{johansson2019support,zhao2019learning}. RUDA and CDAN perform similarly even in this setting. It is interesting to note that the weights improve significantly RUDA results  (+1.7\% on \textbf{Office-31} and +16.0\% on \textbf{Digits} both in average) while CDAN seems less impacted by them (+0.2\% on \textbf{Office-31} and +10.0\% on \textbf{Digits} both in average).

\paragraph{Should we use weights?} To observe a significant benefit of weights, we had to explore situations with strong label shift \textit{e.g.} $5\%$ and $10\% \times [0 \sim 5]$ for the \textbf{Digits} dataset. Apart from this cases, weights bring small gain (\textit{e.g.} + 1.7\% on \textbf{Office-31} for RUDA) or even degrade marginally adaptation. Understanding why RUDA and CDAN are able to address small label shift, without weights, is of great interest for the development of more robust UDA.

\section{Related work}
This paper makes several contributions, both in terms of theory and algorithm. Concerning theory, our bound provides a risk suitable for domain adversarial learning with weighting strategies. Existing theories for non-overlapping supports \cite{ben2010theory,mansour2009domain} and importance sampling \cite{cortes2010learning,quionero2009dataset} do not explore the role of representations neither the aspect of adversarial learning. In \cite{ben2007analysis}, analysis of representation is conducted and connections with our work is discussed in the paper. The work \cite{johansson2019support} is close to ours and introduces a distance which measures support overlap between source and target distributions under covariate shift. Our analysis does not rely on such assumption, its range of application is broader. 

Concerning algorithms,  the covariate shift adaptation has been well-studied in the literature \cite{huang2007correcting,gretton2009covariate,sugiyama2007covariate}. Importance sampling to address label shift has also been investigated \cite{storkey2009training}, notably  with kernel mean matching \cite{zhang2013domain} and Optimal Transport \cite{redko2018optimal}. Recently, a scheme for estimating labels distribution ratio with consistency guarantee has been proposed \cite{lipton2018detecting}. Learning domain invariant representations has also been investigated in the fold of \cite{ganin2015unsupervised,long2015learning} and mainly differs by the metric chosen for comparing distribution of representations. For instance, metrics are domain adversarial (Jensen divergence) \cite{ganin2015unsupervised,long2018conditional}, IPM based such as MMD \cite{long2015learning,long2017deep} or Wasserstein \cite{bhushan2018deepjdot,shen2018wasserstein}. Our work provides a new theoretical support for these methods since our analysis is valid for any IPM.

Using both weights and representations is also an active topic, namely for Partial Domain Adaptation (PADA) \cite{cao2018partial}, when target classes are strict subset of the source classes, or Universal Domain Adaptation \cite{you2019universal}, when new classes may appear in the target domain. \cite{cao2018partial} uses an heuristic based on predicted labels for re-weighting representations. However, it assumes they have a good classifier at first in order to obtain cycle consistent weights. \cite{zhang2018importance} uses a second discriminator for learning weights, which is similar to \cite{cao2018unsupervised}. Applying our framework to Partial DA and Universal DA is an interesting future direction. Our work shares strong connections  with \cite{combes2020domain} (authors were not aware of this work during the elaboration of this paper) which uses consistent estimation of true labels distribution from \cite{lipton2018detecting}. We suggest a very similar empirical evaluation and we also investigate the effect of weights on CDAN loss \cite{long2018conditional} with a different weighting scheme since our approach computes weights in the representation space. All these works rely on an assumption at some level, \textit{e.g.} \textit{Generalized Label Shift} in \cite{combes2020domain}, when designing weighting strategies.  Our discussion on the role of inductive design of weights may provide a new theoretical support for these approaches.

\section{Conclusion}

The present work introduces a new bound of the target risk which unifies weights and representations in UDA. We conduct a theoretical analysis of the role of inductive bias when designing both weights and the classifier. In light of this analysis, we propose a new learning procedure which leverages two weak inductive biases, respectively on weights and the classifier. To the best of our knowledge, this procedure is original while being close to straightforward hybridization of existing methods. We illustrate its effectiveness on two benchmarks. The empirical analysis shows that weak inductive bias can make adaptation more robust even when stressed by strong label shift between source and target domains.  This work leaves room for in-depth study of stronger inductive bias by providing both theoretical and empirical foundations.

\section*{Acknowledgements}
Victor Bouvier is funded by Sidetrade and ANRT (France) through a CIFRE collaboration with CentraleSupélec. Authors thank the anonymous reviewers for their insightful comments for improving the quality of the paper. This work was performed using HPC resources from the “Mésocentre” computing center of CentraleSupélec and École Normale Supérieure Paris-Saclay supported by CNRS and Région Île-de-France (\url{http://mesocentre.centralesupelec.fr/}).

\bibliographystyle{splncs04} 
\bibliography{main}

\newpage
\appendix 
 
\section{Proofs}
\label{proof}
We provide full proof of bounds and propositions presented in the paper. 

\subsection{Proof of bound \ref{TVBound_no_w}}
\label{proof:TVBound_no_w}
We give a proof of bound \ref{TVBound_no_w} which states: 
\begin{equation}
    \varepsilon_T(g\varphi) \leq \varepsilon_{S}(g\varphi)  + 6\cdot\mathrm{INV}(\varphi) +  
    2\cdot \mathrm{TSF}(\varphi)  + \varepsilon_T(\mathbf f_T\varphi)
\end{equation}
First, we prove the following lemma:
\begin{bound}[Revisit of theorem \ref{ben_david}]
\label{lemma}
$\forall g \in \mathcal G$:
\begin{equation}
    \varepsilon_T(g\varphi) \leq \varepsilon_S(g\varphi) + d_{\mathcal F_C}(\varphi) + \varepsilon_T(\mathbf f_S\varphi, \mathbf f_T\varphi) + \varepsilon_T(\mathbf f_T\varphi)
\end{equation}
\end{bound}

\begin{proof}
This is simply obtained using triangular inequalites:
\begin{align*}
    \varepsilon_T(g\varphi) &\leq \varepsilon_T(\mathbf f_T\varphi) + \varepsilon_T(g\varphi, \mathbf f_T\varphi) \\
    & \leq  \varepsilon_T(\mathbf f_T\varphi) + \varepsilon_T(g\varphi,\mathbf f_S\varphi) +  \varepsilon_T(\mathbf f_S\varphi,\mathbf f_T\varphi) 
\end{align*}
Now using (A3) ($\mathbf f_S \in \mathcal F_C$) :

\begin{equation}
    \left | \varepsilon_T(g\varphi,\mathbf f_S\varphi) - \varepsilon_S(g \varphi, \mathbf f_S\varphi) \right | \leq \sup_{\mathbf f \in \mathcal F_C} \left | \varepsilon_T(g\varphi,\mathbf f\varphi) - \varepsilon_S(g\varphi , \mathbf f\varphi) \right | = d_{\mathcal F_C}(\varphi) 
   \end{equation}
which shows that:   $ \varepsilon_T(g\varphi) \leq \varepsilon_S(g\varphi, \mathbf f_S) + d_{\mathcal F_C}(\varphi) + \varepsilon_T(\mathbf f_S\varphi,\mathbf f_T\varphi) + \varepsilon_T(\mathbf f_T\varphi)$ and we use the property of conditional expectation  $\varepsilon_S(g\varphi, \mathbf f_S\varphi) \leq \varepsilon_S(g\varphi)$. $\square$
\end{proof}

Second, we bound $d_{\mathcal F_C}(\varphi)$.

\begin{proposition} 
\label{d_c_inv}
$d_{\mathcal F_C}(\varphi) \leq 4 \cdot \mathrm{INV}(\varphi)$.
\end{proposition}

\begin{proof}
We remind  that $d_{\mathcal F_C}(\varphi) = \sup_{\mathbf f, \mathbf f' \in \mathcal F_C}| \mathbb E_S[ || \mathbf f\varphi(X) - \mathbf f'\varphi(X)||^2] - \mathbb E_T[ ||\mathbf f\varphi(X) - \mathbf f'\varphi(X)||^2] |$. Since (A1) ensures $\mathbf f'\in\mathcal F_C$, $-\mathbf f' \in \mathcal F_C$, then $\frac 1 2 (\mathbf f- \mathbf f') = \mathbf f'' \in \mathcal F_C$ and finally $d_{\mathcal F_C}(\varphi) \leq 4 \sup_{\mathbf f'' \in \mathcal F_C}| \mathbb E_S[ ||\mathbf f'' \varphi||^2] -  \mathbb E_T[ ||\mathbf f'' \varphi||^2] |$. Furthermore, (A2) ensures that $\{ ||\mathbf f''\varphi||^2 \} \subset \{f\varphi, f\in \mathcal F\}$ which leads finally to the announced result. $\square$
\end{proof}

Third, we bound $\varepsilon_T(\mathbf f_S\varphi, \mathbf f_T\varphi)$.

\begin{proposition} $\varepsilon_T(\mathbf f_S\varphi, \mathbf f_T\varphi) \leq 2\cdot \mathrm{INV}(\varphi) + 2\cdot \mathrm{TSF}(\varphi)$. \end{proposition}

\begin{proof}
We note $\Delta = \mathbf f_T - \mathbf f_S$ and we omit $\varphi$  for the ease of reading
\begin{align*}
 \varepsilon_T(\mathbf f_S,\mathbf f_T) & = \mathbb E_T[||\Delta||^2] & \\
 & =  \mathbb E_T[\mathbf f_T \cdot \Delta] - \mathbb E_T[\mathbf f_S \cdot \Delta]  \\
 & =  \left (\mathbb E_T[\mathbf f_T \cdot \Delta ] - \mathbb E_S[\mathbf f_S \cdot \Delta] \right )+  \left (\mathbb E_S[ \mathbf  f_S \cdot \Delta ] - \mathbb E_T[ \mathbf f_S \cdot \Delta] \right)
\end{align*}
Since $\mathbf f_T$ does not intervene in $\mathbb E_S[ \mathbf  f_S \cdot \Delta ] - \mathbb E_T[ \mathbf f_S \cdot \Delta] $, we show this term behaves similarly than $\mathrm{INV}(\varphi)$. First, 

\begin{align}
    \mathbb E_S[\mathbf f_S \cdot \Delta] - \mathbb E_T[\mathbf f_S \cdot \Delta ] & \leq 2 \sup_{\mathbf f \in \mathcal F_{C}} \mathbb E_S[\mathbf f_S \cdot \mathbf f] - \mathbb E_T[\mathbf f_S \cdot \mathbf f] \tag{Using (A1)}  \\
    & \leq 2 \sup_{\mathbf f, \mathbf f' \in \mathcal F_{C}} \mathbb E_S[\mathbf f' \cdot \mathbf f] - \mathbb E_T[\mathbf f' \cdot \mathbf f] \tag{Using (A3)}  \\
    & \leq 2  \sup_{f \in \mathcal F} \mathbb E_S [f] - \mathbb E_T[f] \tag{Using (A2)} \\
    & = 2\cdot \mathrm{INV}(\varphi)
\end{align}

Second, 
\begin{align}
    \mathbb E_T[\mathbf f_T \cdot \Delta] - \mathbb E_S[\mathbf f_S \cdot \Delta] \leq 2 \sup \mathbb E_T[\mathbf f_T \cdot \mathbf f] - \mathbb E_S[ \mathbf f_S \cdot \mathbf f ] = 2 \cdot \mathrm{TSF}(\varphi) \tag{Using (A1)}
\end{align}
which finishes the proof. $\square$
\end{proof}

Note that the fact $\mathbf f_S, \mathbf f_T \in \mathcal F_C$ is not of the utmost importance since we can bound: 
 \begin{equation}
     \varepsilon_T(g\varphi) \leq \varepsilon_S(g\varphi, \hat{\mathbf f}_S) + d_{\mathcal F_C}(\varphi) + \varepsilon_T( \hat{\mathbf f}_S, \hat{\mathbf f}_T) + \varepsilon_T( \hat{\mathbf f}_T)
 \end{equation} 
 where $ \hat{\mathbf f}_D = \arg \min_{\mathbf f \in \mathcal F_C} \varepsilon_D(\mathbf f)$. The only change emerges in the transferability error which becomes: 
 \begin{equation}
     \mathrm{TSF}(w, \varphi ) = \sup_{\mathbf f \in \mathcal F_C} \mathbb E_T[\hat{\mathbf f}_T \varphi \cdot \mathbf f \varphi] - \mathbb E_S[\hat{\mathbf f}_S \varphi  \cdot \mathbf f \varphi ]
 \end{equation}
 
 \subsection{Proof of the new invariance transferability trade-off}
 \label{proof:new_trade_off}
 
 \begin{proposition}
Let $\psi$ a representation which is a richer feature extractor than $\varphi$: $\mathcal F\circ\varphi \subset \mathcal F \circ \psi$ and $\mathcal F_C\circ\varphi \subset \mathcal F_C \circ \psi$. Then, $\varphi$ is more domain invariant than $\psi$:
\begin{equation}
    \mathrm{INV}(\varphi) \leq \mathrm{INV}(\psi) \mbox{ while } \varepsilon_T (f_T^\psi \psi) \leq \varepsilon_T  (f_T^\varphi  \varphi)
\end{equation}
where $f_T^\varphi(z) = \mathbb E_T[Y|\varphi(X) = z]$ and $f_T^\psi(z) = \mathbb E_T[Y|\psi(X)=z]$.
\end{proposition}

\begin{proof}
First, $ \mathrm{INV}(\varphi) \leq \mathrm{INV}(\psi) $ a simple property of the supremum. The definition of the conditional expectation leads to $\varepsilon_T(\mathbf f_T^\psi\psi) = \inf_{f \in \mathcal F_m} \varepsilon_T(f\psi)$ where $\mathcal F_m$ is the set of measurable functions. Since (A3) ensures that $\mathbf f_T^\psi \in \mathcal F_C$ then $\varepsilon_T(\mathbf f_T^\psi\psi) = \inf_{\mathbf f \in \mathcal F_C}\varepsilon_T(\mathbf f \psi) $. The rest is simply the use of the property of infremum. $\square$
\end{proof}

\subsection{Proof of the tightness of bound \ref{TVBound_no_w}}
\label{proof:tightness_TVBound_no_w}
\begin{proposition}
$\mathrm{INV}(\varphi) + \mathrm{TSF}(\varphi) = 0$ if and only if $
    p_S(y,z) = p_T(y,z)$.
\end{proposition}
\begin{proof}
First, $\mathrm{INV}(\varphi) = 0$ implies $p_T(z) = p_S(z)$ which is a direct application of (A4). Now $\mathrm{TSF}(\varphi) = \sup_{\mathbf f \in \mathcal F_C} \mathbb E_S[\mathbf f_S(Z) \cdot \mathbf f(Z) ]- \mathbb E_T[\mathbf f_T(Z) \cdot \mathbf f(Z) ] =\sup_{\mathbf f \in \mathcal F_C} \mathbb E_S[\mathbf f_S(Z) \cdot \mathbf f(Z) ]- \mathbb E_S[\mathbf f_T(Z) \cdot \mathbf f(Z) ] = \sup_{\mathbf f \in \mathcal F_C} \mathbb E_S[(\mathbf f_S - \mathbf f_T )(Z) \cdot \mathbf f (Z)]$. For the particular choice of $\mathbf f = \frac 1 2 (\mathbf f_S - \mathbf f_T)$ leads to $\mathbb E_S[||\mathbf f_S - \mathbf f_T ||^2]$ then $\mathbf f_s = \mathbf f_T$, $p_S$ almost surely. All combined leads to $p_S(y,z) = p_T(y,z)$. The converse is trivial. Note that $\mathrm{TSF}(\varphi)= 0$ is enough to show $p_S(z) = p_T(z)$ by choosing $\mathbf f(z) = (f(z) ,..., f(z))$ ($C$ times $f(z)$) and $Y\cdot \mathbf f(Z) = f(Z)$ then $\mathrm{TSF}(\varphi)\geq \sup_{f\in\mathcal F} \mathbb E_S[f(Z)] - \mathbb E_T[f(Z)]$. $\square$
\end{proof} 

\subsection{Proof of the tightness of bound \ref{TV_bound_with_w}}
\label{proof:tightness_TV_bound_with_w}
\begin{proposition}
$\mathrm{INV}(w,\varphi) + \mathrm{TSF}(w,\varphi) = 0$ if and only if $
    w(z) = \frac{p_T(z)}{p_S(z)}$ and $\mathbb E_T[Y|Z=z] = \mathbb E_S[Y|Z=z]$.
\end{proposition}
\begin{proof}
First, $\mathrm{INV}(w, \varphi) = 0$ implies $p_T(z) = w(z) p_S(z)$ then which is a direct application of (A4). Now $\mathrm{TSF}(w, \varphi) = \sup_{\mathbf f \in \mathcal F_C} \mathbb E_S[ w(z) \mathbf f_S(Z) \cdot \mathbf f(Z) ]- \mathbb E_T[\mathbf f_T(Z) \cdot \mathbf f(Z) ] =\sup_{\mathbf f \in \mathcal F_C} \mathbb E_S[w(z) \mathbf f_S(Z) \cdot \mathbf f(Z) ]- \mathbb E_S[w(z) \mathbf f_T(Z) \cdot \mathbf f(Z) ] = \sup_{\mathbf f \in \mathcal F_C} \mathbb E_S[(\mathbf f_S - \mathbf f_T )(Z) \cdot \mathbf f (Z)]$. For the particular choice of $\mathbf f = \frac 1 2 (\mathbf f_S - \mathbf f_T)$ leads to $\mathbb E_S[||\mathbf f_S - \mathbf f_T ||^2]$ then $\mathbf f_s = \mathbf f_T$, $p_T$ almost surely. The converse is trivial. $\square$
\end{proof}

\subsection{Proof of bound \ref{IB}} 
\label{proof:IB}

\begin{bound}[Inductive Bias and Guarantee]
Let $\varphi \in \Phi$ and $w: \mathcal Z\to \mathbb R^+$ such that $\mathbb E_S[w(z)]=1$ and a $\beta-$strong inductive classifier $\tilde g$, then:
\begin{align}
    \notag \varepsilon_T(\tilde g\varphi) \leq \frac{\beta}{1- \beta} \left ( \varepsilon_{w\cdot S}(g_{w\cdot S}\varphi)  + 6\cdot\mathrm{INV}(w, \varphi) +  
    2\cdot \widehat{\mathrm{TSF}}(w, \varphi, \tilde g) + \varepsilon_T(\mathbf f_T\varphi) \right) 
\end{align}
\end{bound}

\begin{proof} We prove the bound in the case where $w=1$, the general case is then straightforward. First, we reuse bound \ref{lemma} with a new triangular inequality involving the inductive classifier $\tilde g$:
\begin{equation}
    \varepsilon_T(g\varphi) \leq \varepsilon_S(g\varphi) + d_{\mathcal F_C}(\varphi) + \varepsilon_T(\mathbf f_S\varphi ,\tilde g \varphi) + \varepsilon_T(\tilde g\varphi, \mathbf f_T \varphi) +  \varepsilon_T(\mathbf f_T\varphi)
\end{equation}
where $\varepsilon_T(\tilde g\varphi, \mathbf f_T\varphi) \leq \varepsilon_T(\tilde g\varphi)$. Now, following previous proofs, we can show that:
\begin{equation}
    \varepsilon_T(\mathbf f_S \varphi, \tilde g\varphi) \leq 2\cdot \widehat{\mathrm{TSF}}(\varphi, \tilde g) + 2 \cdot \mathrm{INV}(\varphi)
\end{equation}
Then,
\begin{equation}
    \varepsilon_T(g\varphi) \leq \varepsilon_{S}(g\varphi)  + 6\cdot\mathrm{INV}(\varphi) +  
    2\cdot \widehat{\mathrm{TSF}}(\varphi, \tilde g)  + \varepsilon_T(\tilde g\varphi) + \varepsilon_T(\mathbf f_T\varphi)
\end{equation}
This bound is true for any $g$ and in particular for the best source classifier we have:
\begin{equation}
    \varepsilon_T(g_{S}\varphi) \leq \varepsilon_{ S}(g_{S}\varphi)  + 6\cdot\mathrm{INV}(w, \varphi) +  
    2\cdot \widehat{\mathrm{TSF}}(w, \varphi, \tilde g)  + \varepsilon_T(\tilde g\varphi)  + \varepsilon_T(\mathbf f_T\varphi)
\end{equation}
then the assumption of $\beta-$strong inductive bias is $\varepsilon_T(\tilde g\varphi) \leq \beta \varepsilon_T(g_S\varphi)$ which leads to 
\begin{equation}
    \varepsilon_T(g_{S}\varphi) \leq \varepsilon_{ S}(g_{S}\varphi)  + 6\cdot\mathrm{INV}(w, \varphi) +  
    2\cdot \widehat{\mathrm{TSF}}(w, \varphi, \tilde g)  + \beta \varepsilon_T(g_S\varphi) + \varepsilon_T(\mathbf f_T\varphi)
\end{equation}
Now we have respectively $\varepsilon_T(g_{S}\varphi)$ and $\beta \varepsilon_T(g_S\varphi)$ at left and right of the inequality. Since $1-\beta > 0$, we have:
\begin{equation}
    \varepsilon_T(g_{S}\varphi) \leq \frac{1}{1-\beta} \left (\varepsilon_{ S}(g_{S}\varphi)  + 6\cdot\mathrm{INV}(w, \varphi) +  
    2\cdot \widehat{\mathrm{TSF}}(w, \varphi, \tilde g)  + \varepsilon_T(\mathbf f_T\varphi) \right )
\end{equation}
And finally:
\begin{equation}
    \varepsilon_T(\tilde g \varphi) \leq \beta \varepsilon_T(g_{S}\varphi) \leq \frac{\beta}{1-\beta} \left (\varepsilon_{ S}(g_{S}\varphi)  + 6\cdot\mathrm{INV}(w, \varphi) +  
    2\cdot \widehat{\mathrm{TSF}}(w, \varphi, \tilde g)   + \varepsilon_T(\mathbf f_T\varphi) \right )
\end{equation}
fnishing the proof. $\square$
\end{proof}

\subsection{MinEnt \cite{grandvalet2005semi} is a lower bound of transferability}
\label{proof:MinENT}
\begin{proof} We consider a label smooth classifier $g \in \mathcal G$ \textit{i.e.} there is $0<\alpha<1$ such that:
\begin{equation}
    \frac{\alpha}{C-1} \leq g(z) \leq 1 -\alpha
\end{equation}
and we note $Y = g\varphi(X)$. One can show that:
\begin{equation}
    \log\left (\frac{\alpha}{C-1}\right ) \leq \log (g(z)) \leq \log(1 -\alpha)
\end{equation}
and finally:
\begin{equation}
    1 \geq \frac{1}{\log(\frac{\alpha}{C-1})}\log (g(z)) \geq \frac{1}{\log(\frac{\alpha}{C-1})}\log(1 -\alpha)\geq 0
\end{equation}

We choose as particular $\mathbf f$, $\mathbf f(z) = - \eta \log(g(z)) $ with $\eta =  - \log(\frac{\alpha}{|\mathcal Y|-1})^{-1} >0$. The coefficient $\eta$ ensures that $\mathbf f(z) \in [0,1]$ to make sure $\mathbf f \in \mathcal F_C$. We have the following inequalities: 

\begin{align*}
    \widehat{\mathrm{TSF}}(w,\varphi, g) & \geq \eta \cdot  \left ( \mathbb E_T[ - g(Z)  \cdot \log (g(Z)) ]  -  \mathbb E_{w\cdot S}[ - Y \log (g(Z))]\right ) \\
   & \geq \eta \cdot \left (H_T(\hat Y|Z) - \mathrm{CE}_{w \cdot S}(Y, g(Z))\right )
\end{align*}
\end{proof}
 Interestingly, the cross-entropy is involved. Then, when using $\mathrm{CE}_{w \cdot S}(Y, g(Z))$ as a proxy of $\varepsilon_{w\cdot S}(g\varphi)$, we can observe the following lower bound: 
 \begin{equation}
     \mathrm{CE}_{w \cdot S}(Y, g(Z)) + \widehat{\mathrm{TSF}}(w,\varphi, g) \geq (1 - \eta)  \cdot \mathrm{CE}_{w \cdot S}(Y, g(Z)) + \eta \cdot  H_T(\hat Y|Z)
 \end{equation}
 which is a trade-off between minimizing the cross-entropy in the source domain while maintaining a low entropy in prediction in the target domain. 
 
 \begin{figure}
     \centering
     \includegraphics[scale=0.5]{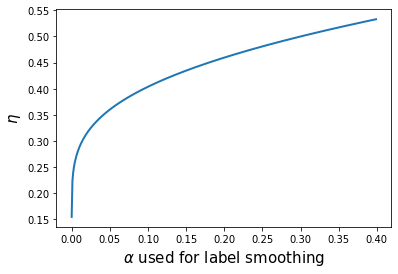}
     \caption{We set $C=31$ which is the number of classes in \textbf{Office31}. Label smoothing $\alpha$ leads naturally to a coefficient $\eta$ which acts as a trade-off between cross-entropy minimization in the source domain and confidence in predictions in the target domain. This result follows a particular choice of the critic function in the transferability error introduced in this paper.}
     \label{fig:alpha_lambda}
 \end{figure}
 
\subsection{Proof of the inductive design of weights}
\label{proof:inductive_weight}
\begin{proposition}[Inductive design of $w$ and invariance] Let $\psi: \mathcal Z \to \mathcal Z'$ such that $\mathcal F \circ \psi \subset \mathcal F$ and $\mathcal F_C\circ \psi\subset \mathcal F_C$. Let $w:\mathcal Z' \to \mathbb R^+$ such that $\mathbb E_S[w(Z')]=1$ and we note $Z':= \psi(Z)$. Then, $\mathrm{INV}(w, \varphi) = \mathrm{TSF}(w,\varphi) = 0$ if and only if:
\begin{equation}
    w(z') = \frac{p_T(z')}{p_S(z')} ~~\mbox{ and } ~~ p_S(z|z') = p_T(z|z')
\end{equation}
while both $\mathbf f_S^\varphi = \mathbf f_T^\varphi$ and $\mathbf f_S^\psi = \mathbf f_T^\psi$.
\end{proposition}

\begin{proof}
First, 

\begin{align}
    \mathrm{INV}(w,\varphi ) & = \sup_{f\in \mathcal F} \mathbb E_S[w(Z') f(Z)] - \mathbb E_T[f(z)] \\
    & \geq \sup_{f\in \mathcal F} \mathbb E_S[w(Z') f\circ \psi (Z)] - \mathbb E_T[f\circ \psi (z)] \tag{$\mathcal F \circ \psi \subset \mathcal F$}  \\
    & = \sup_{f\in \mathcal F} \mathbb E_S[w(Z') f (Z')] - \mathbb E_T[f (z')] = 0 \tag{$Z' = \psi(Z)$}
\end{align}
which leads to $w(z') p_S(z') = p_S(z')$ which is $w(z') = p_T(z') / p_S(z')$. Second, $\mathrm{INV}(w,\varphi)=0$ also implies that $w(z') p_S(z) = p_T(z)$:
\begin{equation}
    w(z') = \frac{p_T(z)}{p_S(z)} = \frac{p_T(z|z')}{p_S(z|z')} \frac{p_T(z')}{p_S(z')} = \frac{p_T(z|z')}{p_S(z|z')} w(z')
\end{equation}
then $p_T(z|z') = p_S(z|z')$. Finally,
\begin{align}
    \mathrm{TSF}(w,\varphi) &= \sup_{\mathbf f \in \mathcal F_C} \mathbb E_S[w(Z') Y \cdot \mathbf f(Z) ] - \mathbb E_T[Y \cdot \mathbf f(Z) ] \\
    & = \sup_{\mathbf f \in \mathcal F_C} \mathbb E_{Z'\sim p_S}\left [w(Z') \mathbb E_{Z|Z'\sim p_S} [Y \cdot \mathbf f(Z) ]\right] - \mathbb E_{Z' \sim p_T}\left [ \mathbb E_{Z|Z' \sim p_T}[Y \cdot \mathbf f(Z) ]\right] \\
    & = \sup_{\mathbf f \in \mathcal F_C} \mathbb E_{Z'\sim p_S}\left [w(Z') \mathbb E_{Z|Z'\sim p_S} [Y \cdot \mathbf f(Z) ]\right] - \mathbb E_{Z' \sim p_T}w(Z')\left [ \mathbb E_{Z|Z' \sim p_S}[Y \cdot \mathbf f(Z) ]\right] \tag{$w(z') p_S(z') = p_T(z')$}\\
    & = \sup_{\mathbf f \in \mathcal F_C} \mathbb E_{Z'\sim p_S}\left [w(Z') \left ( \mathbb E_{Z|Z'\sim p_S} [Y \cdot \mathbf f(Z) ] -  \mathbb E_{Z|Z'\sim p_T} [Y \cdot \mathbf f(Z) ] \right) \right] \\
    & = \sup_{\mathbf f \in \mathcal F_C} \mathbb E_{Z'\sim p_S}\left [w(Z') \left ( \mathbb E_{Z|Z'\sim p_S} [\mathbf f_S(Z) \cdot \mathbf f(Z) - \mathbf f_T(Z) \cdot \mathbf f(Z)] \right) \right] \tag{$p_S(z|z') = p_T(z|z')$} \\
    & = \sup_{\mathbf f \in \mathcal F_C} \mathbb E_{Z'\sim p_S}\left [w(Z') \left ( \mathbb E_{Z|Z'\sim p_S} [\mathbf f_S(Z) \cdot \mathbf f(Z) - \mathbf f_T(Z) \cdot \mathbf f(Z)] \right) \right] \tag{$p_S(z|z') = p_T(z|z')$} \\
    & \geq 2 \mathbb E_{Z'\sim p_S}\left [w(Z') \left ( \mathbb E_{Z|Z'\sim p_S} [|| \mathbf f_S(Z) - \mathbf f_T(Z)||^2] \right) \right] \\ 
    & \geq 2 \mathbb E_{Z'\sim p_T}\left [ \left ( \mathbb E_{Z|Z'\sim p_T} [|| \mathbf f_S(Z) - \mathbf f_T(Z)||^2] \right) \right] \\
    & \geq 2 \mathbb E_{Z'\sim p_T}\left [ \left ( \mathbb E_{Z|Z'\sim p_T} [|| \mathbf f_S(Z) - \mathbf f_T(Z)||^2] \right) \right]\\
    & \geq 2 \mathbb E_{Z\sim p_T}\left [ || \mathbf f_S(Z) - \mathbf f_T(Z)||^2] \right]
\end{align}
Which leads to $\mathbf f_S(z) = \mathbf f_T(z) $, $p_T(z)$ almost surely, then $\mathbb E_T[Y|Z] = \mathbb E_S[Y|Z]$ for $Z\sim p_T$. Now we finish by observing that: 
\begin{align}
    \mathrm{TSF}(w,\varphi) &= \sup_{\mathbf f \in \mathcal F_C} \mathbb E_S[w(Z') Y \cdot \mathbf f(Z) ] - \mathbb E_T[Y \cdot \mathbf f(Z) ] \\
    & \geq  \sup_{\mathbf f \in \mathcal F_C} \mathbb E_S[w(Z') Y \cdot \mathbf f\circ \psi (Z) ] - \mathbb E_T[Y \cdot \mathbf f\circ \psi (Z) ] \\
    & \geq  \sup_{\mathbf f \in \mathcal F_C} \mathbb E_S[w(Z') Y \cdot \mathbf f(Z') ] - \mathbb E_T[Y \cdot \mathbf f(Z')]
\end{align}
which leads to $\mathbb E_S[Y|Z'] = \mathbb E_T[Y|Z']$ for $Z'\sim p_T$. The converse is trivial. $\square$
\end{proof}

\section{CDAN, DANN and TSF: An open dicussion.}
\label{open_dicussion}
In CDAN \cite{long2018conditional}, authors claims to align conditional $Z | \hat Y$, by exposing the multi-linear mapping of $\hat Y$ by Z, hence its name of Conditional Domain Adversarial Network. Here, we show this claim can be theoretically misleading:

\begin{proposition} If $\mathbb E[\hat Y|Z]$ is conserved across domains, \textit{i.e.} $g$ is conserved, and $\mathcal D$ and $\mathcal D_{\otimes}$ are infinite capacity set of discriminators, this holds:
\begin{equation}
    \mathrm{DANN}(\varphi) = \mathrm{CDAN}(\varphi) 
\end{equation}
\end{proposition}

\begin{proof}
First, let $d_{\otimes} \in \mathcal D_{\otimes}$. Then, for any $(\hat y, z) \sim p_S$ (similarly $\sim p_T$), $d(\hat y \otimes  z) = d(g(z) \otimes z)$ since $\hat y = g(z) = \mathbb E[\hat Y|Z=z]$ is conserved across domains. Then $\tilde d : z\mapsto d_{\otimes}(g(z)\otimes z)$ is a mapping from $\mathcal Z$ to $[0,1]$. Since $\mathcal D$ is the set of infinite capacity discriminators, $\tilde d \in \mathcal D$. This shows $ \mathrm{CDAN}(\varphi) \leq \mathrm{DANN}(\varphi)$. Now we introduce $T: \mathcal Y \otimes \mathcal Z \to \mathcal Z$ such that $T(y\otimes z) = \sum_{1 \leq  c \leq |\mathcal Y|} y_c (y \otimes z)_{c r : (c+1) r} =z$ where $r= \mathrm{dim}(Z)$. The ability to reconstruct $z$ from $\hat y \otimes z$ results from $\sum_c y_c = 1$. This shows that $\mathcal D_{\otimes} \circ T = \mathcal D$ and finally $\mathrm{CDAN}(\varphi) \geq \mathrm{DANN}(\varphi)$ finishing the proof.
\end{proof}

This proposition follows two key assumptions. The first is to assume that we are in context of infinite capacity discriminators of both $\mathcal Z$ and $\mathcal Y \otimes \mathcal Z$. This assumption seems reasonable in practice since discriminators are multi-layer perceptrons. The second is to assume that $\mathbb E[\hat Y |Z]$ is conserved across domains. Pragmatically, the same classifier is used in both source and target domains which is verified in practice. Despite the empirical success of CDAN, there is no theoretical evidence of the superiority of CDAN with respect to DANN for UDA. However, our discussion on the role of inductive design of classifiers is an attempt to explain the empirical superiority of such strategies.

\section{More training details}

\subsection{From IPM to Domain Adversarial Objective}
\label{tv_to_da}
While our analysis holds for IPM, we recall the connections with  $\mathsf f-$divergence, where domain adversarial loss is a particular instance, for comparing distributions. This connection is motivated by the furnished literature on adversarial learning, based on domain discriminator, for UDA. This section is then an informal attempt to transport our theoretical analysis, which holds for IPM, to $\mathsf f-$divergence. Given $\mathsf f$ a function defined on $\mathbb R^+$, continuous and convex, the $\mathsf f-$divergence between two distributions $p$ and $q$: $\mathbb E_p[\mathsf f(p/q)]$, is null if and only if $p=q$. Interestingly, $\mathsf f-$divergence admits a 'IPM style' expression $\mathbb E_p[\mathsf f(p/q)] = \sup_f \mathbb E_p[f] - \mathbb E_q[\mathsf f^\star (f)]$ where $\mathsf f^\star$ is the convex conjugate of $\mathsf f$. It is worth noting it is not a IPM expression since the critic is composed by $\mathsf f^\star$ in the right expectation. The domain adversarial loss \cite{ganin2015unsupervised} is a particular instance of $\mathsf f-$divergence (see \cite{bottou2018geometrical} for a complete description in the context of generative modelling). Then, we informally transports our analysis on IPM distance to domain adversarial loss. More precisely, we define:
\begin{align}
  \mathrm{INV}_{\mathrm{adv}}(w,\varphi) &:= \log(2) - \sup_{d \in\mathcal D} \mathbb E_S[ w(Z) \log( d(Z) ) ] + \mathbb E_T[ \log(1-d(Z))] \\
   \mathrm{TSF}_{\mathrm{adv}}(w,\varphi) & := \log(2) - \sup_{\mathbf d \in \mathcal D_{\mathcal Y}} \mathbb E_S[ w(Z) Y \cdot  \log( \mathbf d (Z) ) ] + \mathbb E_T[ Y \cdot \log(1-\mathbf d(Z))]
\end{align}
where $\mathcal D$  is the well-established domain discriminator from $\mathcal Z$ to $[0,1]$, and $\mathcal D_{\mathcal Y}$ is the set of \textit{label domain discriminator} from $\mathcal Z$ to $ [0, 1]^C$.

\subsection{Controlling invariance error with relaxed weights}
\label{training_details:relaxed_weights}
In this section, we show that even if representations are not learned in order to achieve domain invariance, the design of weights allows to control the invariance error during learning. More precisely $w^\star(\varphi) = \arg \min_w \mathrm{INV}(w, \varphi)$ has a closed form when given a  domain discriminator $d$ \textit{i.e.} the following function from the representation space $\mathcal Z$ to $[0,1]$: 
\begin{equation}
    d(z) := \frac{p_S(z)}{p_S(z) + p_T(z)}
\end{equation}
Here, setting $w^\star(z) := (1-d(z)) / d(z) = p_T(z) / p_S(z)$ leads to $w(z) p_S(z) = p_T(z)$ and finally $\mathrm{INV}(w^\star(\varphi), \varphi)= 0$. At early stage of learning, the domain discriminator $d$ has a weak predictive power to discriminate domains. Using exactly the closed form $w^\star(z)$ may degrade the estimation of the transferability error. Then, we suggest to build relaxed weights $\tilde w_d$ which are pushed to $w^\star$ during training. This is done using temperature relaxation in the sigmoid output of the domain discriminator: 
\begin{equation}
    w_d^\tau(z) := \frac{1 -\sigma \left (\tilde d(z) /\tau  \right )}{\sigma\left (\tilde d(z) /\tau \right)}
\end{equation}
where $d(z) = \sigma(\tilde d(z))$; when $\tau \to 1$, $w_d(z, \tau) \to w^\star(z)$.

\subsection{Ablation study of the weight relaxation parameter $\alpha$}
\label{ablation_alpha}
$\alpha$ is the rate of convergence of relaxed weights to optimal weights. We investigate its role on the task U$\to$M. Increasing $\alpha$ degrades adaptation, excepts in the harder case ($5\% \times [0\sim 5]$). Weighting early during training degrades representations alignment. Conversely, in the case $5\% \times [0\sim 5]$, weights need to be introduced early to not learn a wrong alignment. In practice $\alpha=5$ works well (except for $5\%\times[0\sim 5]$ in \textbf{Digits}).
\begin{figure}[H]
\centering
\includegraphics[width=0.5\textwidth]{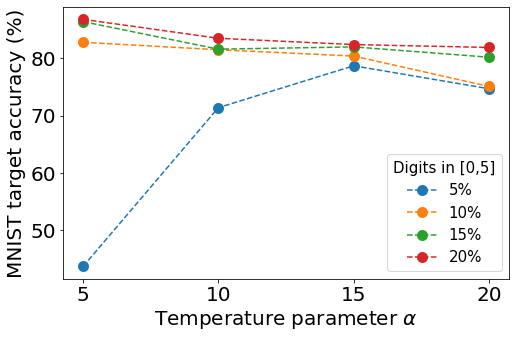}
\caption{\label{fig:blue_rectangle} \small{Effect of $\alpha.$}}
\end{figure}

\subsection{Additional results on Office-Home dataset}
\begin{table}[h!]
\small{
\centering 
\caption{Accuracy ($\%$) on \textbf{Office-Home} based on ReseNet-50. }\label{standard}
\begin{tabular}{|c|cccccccccccc||c|}
\hline
\scriptsize{Method} & \scriptsize{Ar{\tiny{$\to$}}Cl} & \scriptsize{Ar{\tiny{$\to$}}Pr} & \scriptsize{Ar{\tiny{$\to$}}Rw} &\scriptsize{Cl{\tiny{$\to$}}Ar}  & \scriptsize{Cl{\tiny{$\to$}}Pr} & \scriptsize{Cl{\tiny{$\to$}}Rw} & \scriptsize{Pr{\tiny{$\to$}}Ar} & \scriptsize{Pr{\tiny{$\to$}}Cl} & \scriptsize{Pr{\tiny{$\to$}}Rw} & \scriptsize{Rw{\tiny{$\to$}}Ar} & \scriptsize{Rw{\tiny{$\to$}}Cl} & \scriptsize{Rw{\tiny{$\to$}}Pr} &\scriptsize{Avg}  \\
\hline
\scriptsize{ResNet50} & 34.9 & 50.0 & 58.0 & 37.4 & 41.9 & 46.2 & 38.5 & 31.2 & 60.4 & 53.9 & 41.2 & 59.9 & 46.1   \\ \hline 
\scriptsize{DANN} & 45.6 & 59.3 & 70.1 & 47.0 & 58.5 & 60.9 & 46.1 & 43.7 & 68.5 & 63.2 & 51.8 & 76.8 & 57.6 \\
\scriptsize{CDAN} & 49.0 & 69.3 & 74.5 & 54.4 & 66.0 & 68.4 & 55.6 & 48.3 & 75.9 & 68.4& 55.4 & 80.5 & 63.8 \\
\scriptsize{CDAN+E} & 50.7 & \textbf{70.6} & \textbf{76.0} & \textbf{57.6} & \textbf{70.0} & \textbf{70.0} & \textbf{57.4} & \textbf{50.9} & \textbf{77.3} & \textbf{70.9} & 56.7 & \textbf{81.6} & \textbf{65.8} \\
\scriptsize{RUDA} & \textbf{52.0} & 67.1 & 74.4 & 56.8 & 69.5 & 69.8 & 57.3 & \textbf{50.9} & 77.2 & 70.5 & \textbf{57.1} & 81.2 & 64.9\\ 
\hline 

\end{tabular}}
\end{table}

\subsection{Detailed procedure}
The code is available at \url{https://github.com/vbouvier/ruda}.
\label{procedure_detailed}
\begin{algorithm}[h!]
\caption{Procedure for Robust Unsupervised Domain Adaptation}
\label{alg:adversarial_gamma}
\textbf{Input}: Source samples $(x_{S,i}, y_{S,i})_i$, Target samples $(x_{T,i}, y_{T,i})_i$, $(\tau_t)_t$ such that $\tau_t \to 1$, learning rates $(\eta_t)_t$, trade-off $(\alpha_t)_t$ such that $\alpha_t \to 1$, batch-size $b$
\begin{algorithmic}[1] 
\STATE $\theta_g, \theta_\varphi, \theta_d, \theta_{\mathbf d}$ random initialization.
\STATE $t \leftarrow 0$
\WHILE{stopping criterion}
\STATE $\mathcal B_S \sim (x_i^s)$, $\mathcal B_T \sim (x_j^t)$ of size $b$.
\STATE $\theta_d \leftarrow \theta_d - \eta_t \nabla_{\theta_d} \mathcal L_{\mathrm{INV}}(\theta_d |\theta_\varphi ; \mathcal B_S, \mathcal B_T)$
\STATE $\theta_{\mathbf d} \leftarrow \theta_{\mathbf d} - \eta_t \nabla_{\theta_{\mathbf d}} \mathcal L_{\mathrm{TSF}}(\theta_g, \theta_\varphi, \theta_{\mathbf d}| \theta_d, \tau_t)$
\STATE $\theta_\varphi\leftarrow \theta_\varphi - \eta_t \nabla_{\theta_{\varphi}}\left (\mathcal L_c(\theta_g, \theta_\varphi | \theta_d, \tau_t) -  \alpha_t \mathcal L_{\mathrm{TSF}}(\theta_\varphi, \theta_{\mathbf d}|\theta_g,  \theta_d, \tau_t) \right)$
\STATE $\theta_g \leftarrow \theta_g - \eta_t \nabla_{\theta_g} \mathcal L_c(\theta_g, \theta_\varphi|\theta_d, \tau_t)$
\STATE $t \leftarrow t + 1$
\ENDWHILE
\end{algorithmic}
\end{algorithm}

\end{document}